\documentclass[11pt]{article}

\usepackage{booktabs} 
\usepackage{natbib}
\usepackage{xfrac}
\usepackage{algorithm}
\usepackage{amsmath,amsthm,amsfonts}
\usepackage{amsmath}
\usepackage{color}
\usepackage{mathrsfs}
\usepackage{verbatim}
\usepackage{enumitem}
\usepackage{bm}
\usepackage{hyperref}
\usepackage{xr}
\usepackage{tikz}
\usetikzlibrary{positioning}
\usepackage{subcaption}
\usepackage{float}
\usepackage{listings}
\usepackage{array}
\usepackage{makecell}
\usepackage{mathtools}
\usepackage[suppress]{color-edits}
\addauthor{ll}{blue}
\addauthor{kl}{blue}
\addauthor{jk}{blue}
\addauthor{sb}{blue}

\definecolor{DarkBlue}{rgb}{0.1,0.1,0.5}
\hypersetup{
	colorlinks=true,       
	linkcolor=DarkBlue,          
	citecolor=DarkBlue,        
	filecolor=DarkBlue,      
	urlcolor=DarkBlue,          
	pdftitle={},
	pdfauthor={},
}

\usepackage{fancyhdr, lastpage}
\usepackage[vmargin=1.00in,hmargin=1.00in,centering,letterpaper]{geometry}

\newcommand{\todo}{\textcolor[rgb]{1,0,0.5}{To do: }\textcolor[rgb]{0.5,0,1}}

\newcommand{\piv}{\mathrm{piv}}
\newcommand{\pivotal}{\mathrm{pivotal}}
\newcommand{\isnot}{\mathrm{not}}
\newcommand{\pred}{^\mathrm{predict}}
\newcommand{\act}{^\mathrm{act}}

\newcommand{\calM}{\mathcal{M}}

\newcommand{\argmax}{\mathop{\rm argmax}}
\newcommand{\Ind}[1]{\mathbf{1}\{#1\}}

\renewcommand{\Pr}{\mathbb{P}}

\newcommand{\E}{\mathbb{E}}

\newcommand{\calF}{\mathcal{F}}

\newcommand{\calS}{\mathcal{S}}

\newcommand{\calD}{\mathcal{D}}

\newcommand{\calA}{\mathcal{A}}

\newtheorem{theorem}{Theorem}[section]
\newtheorem{definition}{Definition}
\newtheorem{proposition}[theorem]{Proposition}

\begin{document}

\title{On the Actionability of Outcome Prediction}

\author{
Lydia T. Liu\thanks{Correspondence to: lydiatliu@cornell.edu}~\thanks{Cornell University}
\and Solon Barocas\footnotemark[2]~\thanks{Microsoft Research}
\and Jon Kleinberg\footnotemark[2]
\and Karen Levy\footnotemark[2]}

\date{}

\maketitle

\begin{abstract}

Predicting future outcomes is a prevalent application of machine learning in social impact domains. Examples range from predicting student success in education to predicting disease risk in healthcare. Practitioners recognize that the ultimate goal is not just to predict but to act effectively. Increasing evidence suggests that relying on outcome predictions for downstream interventions may not have desired results. 

In most domains there exists a multitude of possible interventions for each individual, making the challenge of taking effective action more acute. 
Even when causal mechanisms connecting the individual's latent states to outcomes is well understood, in any given instance (a specific student or patient), practitioners still need to infer---from budgeted measurements of latent states---which of many possible interventions will be most effective for this individual. With this in mind, we ask: when are accurate predictors of outcomes helpful for identifying the most suitable intervention? 

Through a simple model encompassing actions, latent states, and measurements, we demonstrate that pure outcome prediction rarely results in the most effective policy for taking actions, even when combined with other measurements. 
We find that except in cases where there is a single decisive action for improving the outcome, outcome prediction never maximizes ``action value'', the utility of taking actions.
Making measurements of actionable latent states, where specific actions lead to desired outcomes, considerably enhances the action value compared to outcome prediction, and the degree of improvement depends on action costs and the outcome model. This theoretical analysis underscores the importance of moving beyond generic outcome prediction in policy settings, and of incorporating knowledge of plausible actions and latent states in predictive models to better guide targeted interventions.

\end{abstract}

\section{Introduction}

Artificial intelligence has been used for impact in variety of societal domains, from education to healthcare \citep{shi20survey,tomavsev2020ai}. While many of its applications have focused on \emph{prediction}, such as that of educational outcomes \citep{tamhane2014predicting,lakkaraju2015machine,xu2017progressive} and medical incidents and risk \citep{hosseinzadeh2013assessing, ma2018risk, ballinger2018deepheart,Optum}, practitioners and researchers invariably encounter the question of how to use these predictions for interventions to improve the outcomes that they care about.

Consider the example of predicting student academic performance at the secondary level. In most cases, the goal of building such predictors is to improve the relevant educational \emph{outcome}, academic performance. However, a prediction of a student's future academic performance alone does not improve academic performance unless there is an intervening \emph{action}, such as providing additional tutoring or providing financial support. Students who lack the necessary academic prerequisites may need additional tutoring rather financial support to improve their performance, whereas students who lack the time to complete course work because they are working multiple jobs may need financial aid rather than to be referred for additional tutoring. Therefore, the success of any action depends on a student's \emph{latent state} (so called as we do not know a priori whether the student lacks prerequisites or income). A school official may take \emph{measurements}, such as diagnostic tests, past grades, income survey, to obtain information about the students' latent states. These measurements can be costly, requiring time and labor. Therefore, key questions for the school official include: what should they measure in order to best \emph{predict} the student's future academic performance, and what should they measure in order to best \emph{improve} it? Further, when are these the same measurements, and when are they different?

Many applications of ML/AI for social impact focus on solving prediction problems \citep{kleinberg2015prediction,athey2017beyond} and maximizing prediction accuracy for future outcomes.
Although risk prediction has become ubiquitous in education and other domains, its effectiveness for improving outcomes has been called into question. A recent empirical qualitative study by \citet{liu2023reimagining} on machine learning applications in education found a significant gap between predictions and beneficent interventions.

\begin{quote}
	\emph{
	``You don't improve things by predicting them better." - Education researcher on the value of predicting academic risk} \citep{liu2023reimagining}
\end{quote}

The study of interventions has been fundamental in the social sciences, statistics, and theoretical computer science \citep{rosenbaum1983central, pearl1995causal, rubin2005causal,peters2017elements, hofman2021integrating}. The set of techniques and applications for causal inference and analysis are vast, mostly notably including program evaluation and randomized controlled trials \citep{Stephenson98RCT,deaton2018understanding}, observational studies \citep{rosenbaum2010design} using modern ML techniques \citep{athey2016recursive}, adaptive trial designs \citep{collins2007multiphase,montoya2022efficient}, individual treatment effect and counterfactual inference \citep{shalit2017estimating,lei2021conformal,Bynum_Loftus_Stoyanovich_2023}. 
Prior work informed by causal inference has discussed the gap between predictions and decisions \cite{athey2017beyond}, and how the use of prediction in these cases is predicated upon critical causal assumptions \cite{prosperi2020causal,lundberg2022}.  In the specific application domain of education, despite the prevalence of RCTs and causal analysis on the population impact of interventions \citep{cook2014surprising, yeager2019national}, instance-level targeting and decision making in schools are still often driven by risk scores that predict academic outcomes without incorporating knowledge of plausible interventions \citep{bruce2011track, knowles2015needles, perdomo2023difficult}.

The current work is interested in a question that is at the intersection of the pure prediction and the causal intervention paradigms: when is outcome prediction helpful for interventions at the instance level? Given the ubiquity of predictive tools and significant data infrastructure built around prediction, there is a  need to better understand the limits of predictions when applying them to interventional settings. This work acknowledges the key role of causal inference, while studying a problem at a different scope---that of instance-level predictive intervention, e.g. what helps this patient, what helps this student, assuming that a model of causal effects is available. Unlike in the estimation of heterogeneous treatment effects \citep[see e.g.][]{imai2013estimating,athey2016recursive}, where covariates are assumed to be given, here we investigate the choice of covariates---what to measure, under a constrained budget, in order to predict or intervene. This line of questioning is also related to the theory of \emph{diagnosis} \citep{reiter1987theory,de1987diagnosing} which has a long history in the AI literature; through the current work we bring the analytical framework of diagnosis to bear on current issues in data-driven prediction and decision making in social systems. 

The main contributions of the work are as follows:
We formalize the gap between outcome prediction and intervention in a mathematical framework that combines probabilistic modeling, logical formalism, and a theory of action utility. Our model comprises: latent states of individuals, measurements, outcome, and actions to enact change in the latent states (Section~\ref{sec:model}). We then illustrate the actionability of outcome prediction with a simple numerical example in Section~\ref{sec:comp_ex}. We advance this research agenda in the setting of Boolean functions with a set of theoretical results (Section~\ref{sec:theorems}): we fully characterize the conditions under which outcome prediction can be considered actionable and show that the optimal measurement for outcome prediction is almost never the optimal measurement for outcome improvement, either when used alone or in combination with other measurements. Rather than prescribe how to best perform interventions (e.g. to improve student academic performance), our goal is to precisely describe when prediction necessarily falls short of intervention goals. In Section~\ref{sec:related}, we review further literature in related fields.

\section{Model}\label{sec:model}

Our model of data-driven decision making comprises four key elements: latent states of individuals, measurements,  outcome, and actions to enact change in the latent states. We suppose that an institutional decision maker, whom we refer to as the \emph{planner}, makes measurements for each individual in a population and takes actions based on those measurements for each individual to influence their future outcome. Formally, we describe a graphical model that comprises the following random variables. Each individual is a random draw from the model. 

\begin{itemize}[leftmargin=*]
    \item \emph{States}. There are $s$ latent states $\calS = \{X_1, \cdots, X_s\}$ that are not observed directly, each supported on $\calD_\calS$. \lledit{Each latent state indicates a factor that influences the individual's outcome, and need not be independent of other states. They have a joint distribution.}
    \item \emph{Outcome}. The outcome of interest is $Y$. The distribution of $Y$ depends on latent states, $Y~\sim~g(X_1, \cdots, X_s)$, and is supported on $\calD_Y$. $g$~is known to the planner. \lledit{In other words, we assume the planner knows the structural causal model of how states map to the outcome.}
    \item \emph{Measurements}. Planner chooses from $n$ possible measurements $\calM:=\{M_1, \cdots, M_n\}$. The distribution of $M_i$ depends on States, $M_i \sim f^i(X_1, \cdots, X_s)$ and is supported on $\calD_\calM$. There is a measurement budget of $B$ measurements. \lledit{The planner observes the realized values of $B$ chosen measurements in order to perform subsequent prediction (Section~\ref{subsec:pred}) and intervention (Section~\ref{subsec:int}) tasks.} 
\end{itemize}

\lledit{To reiterate, each graphical model comprises three types of variables: latent states, measurements, and an outcome (see Figure~\ref{fig:two_state_graph} for an illustration). There is a fourth element of the model, which is \emph{actions}.}

\emph{Actions}. After observing the value of measurement(s), the planner takes an action, $a$. Actions change the value of latent states, e.g., the action $a~=~[X_1 \leftarrow 1]$ changes the value of state $X_1$ to $1$.\footnote{We use the notation $[X\leftarrow x]$ to denote the $do$-operation that sets the value of variable $X$ to $x$.} The set of possible actions is denoted $\calA$
    .  The cost function of action is $c:\calA \to [0,C]$, where $C>0$. The cost of taking no action, $a=\emptyset$, is $0$. \lledit{This means that the planner can take a (costly) action on behalf of each individual to modify one of their latent states.}

\lledit{In section~\ref{subsec_probleminstances}, we show how the model can be instantiated across three real world problem domains and give examples of the respective states, outcome, measurements and actions.}

\subsection{Prediction task and prediction value}\label{subsec:pred}

Consider the case where predicting $Y$ is an end in itself. Then the planner wants to choose a measurement $M \subseteq \calM$ such that $|M| \le B$ and $M$ allows the planner to predict $Y$ with the lowest prediction loss (or error) out of all size $B$ measurement sets. Given a hypothesis class $H$, and prediction loss function $\ell$, the planner constructs an optimal predictor $h^*_M$ given $M$:
\begin{equation*}
    h^*_M := \argmax_{h\in H} \E[\ell(h(M), Y)].
\end{equation*}
\lledit{For each observed value (or values) of $M$, the optimal predictor $h^*$ outputs a particular prediction of the individual's outcome $Y$. It minimizes prediction loss over the population.}
We define the \emph{prediction value} of measurement $M$ as:
\begin{equation*}
    V\pred(M) := - \E[\ell(h^*_M(M), Y)].
\end{equation*}
\lledit{The higher the prediction value of a measurement, the more informative it is for predicting the outcome, assuming that an optimal predictor is always available to the planner.}

\subsection{Intervention task and action value}\label{subsec:int}
In most cases, the goal of the planner is not simply to predict $Y$, but to take the best action to attain a more favorable outcome $Y$ for the individual. The notion of the ``best'' action requires the us to define a utility function for actions and outcomes.

Let $Y^a$ denote the outcome variable after an action $a$ has been taken. Action $a$ typically corresponds to a $do$-operation \citep[see e.g.][]{Pearl:2009:CMR:1642718} on the latent states that changes the distribution of $Y$, e.g., if $a = [X_1 \leftarrow 1]$, then $Y^a = Y^{[X_1 \leftarrow 1]}$ is that new random variable for the outcome under $do$-operation that sets the value of state $X_1$ to $1$. Let $u(y)$ denote the utility to the planner of having $Y=y$.

Given any measurement $M \in \calM$, the planner constructs an optimal action policy $a^*_M$  to maximize the net utility of taking action:
 \begin{equation*}
        a^*_M:= \argmax_{a:\calD_{\calM}\to \calA}  \E[u(Y^{a(M)})] -  \E[u(Y)] - \E[c(a(M))].
    \end{equation*}
\lledit{In words, $a^*_M$ maps any value of $M$ to an action that most improves the expected value of $Y$ conditional on the known value of $M$, taking into account action cost.}
    
We define the \emph{action value} of measurement $M$ as:
    \begin{equation*}
        V\act(M) :=  \E[u(Y^{a^*(M)})] -  \E[u(Y)] - \E[c(a^*(M))].
    \end{equation*}

    The first term $\E[u(Y^{a^*(M)})]$ is the expected utility under the action policy $a^*$.  We may write the first term as \begin{equation*}
        \E[u(Y^{a^*(M)})] = \E_M\left[\E[u(Y^{a^*(m)}\mid  M=m]\right].
    \end{equation*} to see that expectation is taken with respect to $Y$ under the $do$-operation (that is, post-action $Y$), as well as with respect to (pre-action) $M$.
    The second term $\E[u(Y)]$ is the expected utility without taking any action. The third term is the expected cost of the the action policy $a^*$.
        \lledit{The higher the action value of a measurement, the most informative it is for taking actions to the improve the outcome in a cost-effective way. }

\subsection{Motivating Problem Instances}\label{subsec_probleminstances}

We now discuss motivating real world problems where the model helps to elucidate the different measurements needed for prediction and for intervention.
We develop the first example on predicting and improving educational outcomes in some detail, and present the second example on actionable genomics for clinical interventions as a brief sketch.

\lldelete{
We now discuss three problem settings from distinct domains---education, healthcare and credit---where the model helps to elucidate the different measurements needed for prediction and for intervention.
The first example is about predicting and intervening on educational outcomes. We develop this example in detail and then sketch two other examples to illustrate problems with a similar structure in different domains.
}

\paragraph{Education and student success} Consider the use of student data and machine learning techniques to predict future educational outcomes, such as the student's risk of adverse academic outcomes in secondary school \citep{lakkaraju2015machine} and academic performance in higher education \citep{bird2021bringing,xu2017progressive,tamhane2014predicting}. In a critical study by \citet{liu2023reimagining}, education researchers that were consulted on the value of making such predictions suggested that the measurements available for making accurate predictions of future educational outcomes (e.g. data on demographic factors, behavioral factors), in the absence of interventions, are not necessarily helpful for selecting interventions to change the outcome. 

The model developed in the previous section formally illustrates such concerns and how they arise from the inherent differences between prediction and interventions at the level of measurement. 

Suppose the planner is a college official whose mandate is to improve student retention rates. We instantiate the following simplified model of student success:
\lledit{
\begin{itemize}[leftmargin=*]
    \item \emph{States}. Latent states $X_1, \cdots, X_s$ may include: $X_1$ (whether the student is overworked at job), $X_2$ (whether the student has grasped the academic prerequisites, e.g. calculus), $X_3$ (a demographic feature, e.g. parental education status), etc. These states tend not to be independent, and tend to be only observable via a measurement.
    \item \emph{Actions.} The corresponding actions are different interventions available to the school official: $[X_1 \leftarrow 1]$ giving financial aid, $[X_2 \leftarrow 1]$ tutoring calculus, etc. There is no corresponding action for $X_3$ as it cannot be modified. 
    \item \emph{Outcome}.  The outcome of interest $Y$ is whether the student returns for sophomore year. It is observed at the start of sophomore year. $Y$ is a function of the states, that is, $Y=f(X_1, X_2, X_3, \cdots)$.
    \item \emph{Measurements}. Some measurements at taken after midterm exams in freshman year. The measurements include $M_1 := X_1$ (student job status),  $M_2:=X_2$ (diagnostic calculus test), $M_3:=X_3$ (demographic), $M_4:=f(X_1, X_2, X_3)$ (midterm grades), etc.
\end{itemize}
}
In this case, knowing $M_3$ (midterm grades) may be very helpful for predicting $Y$, but it is less helpful for determining which costly action (financial aid or tutoring) should be used to intervene on the student's future retention outcome. In the same vein, $M_1$ and $M_2$ can inform whether the student requires a particular intervention, but without $M_3$, they cannot be used to predict $Y$ as accurately, since $Y$ depends on all three latent states. From an education and testing research perspective, diagnostic tests are different from achievement or proficiency tests \citep{alderson2015towards}---even though $M_2$ and $M_4$ are both test results, the former better informs interventions as it diagnoses specific academic areas that benefit from tutoring.

\paragraph{Genomics for clinical decisions} This example is taken from \citet{nelson2013being}, a study of ``actionability'' in the context of clinical sequence. Suppose the planner is a hospital with multiple patients to treat. The outcome of interest is health (e.g. the absence of cancer). Each patient has set of \emph{states} including phenotypes (e.g. whether a patient has a mutated enzyme) and risk factors (e.g. family history of cancer) that together determine their future health outcome. \emph{Measurements} are the genetic sequences of the patient (e.g. whether a patient has a mutation in the anaplastic lymphoma kinase (ALK) gene).

Depending on the type of mutation (e.g. hereditary mutations in pre-symptomatic individuals or non-heritable sporadic mutations), a state may or may not be associated with an \emph{action} that can improve the patient's health outcome: non-heritable mutations in the tumor may be associated with a drug mechanism that that can block the function of the mutated enzyme, whereas gene markers that are associated with future health risks typically cannot be targeted by any particular drug pathway.

\lldelete{
\todo{remove below if space}
\paragraph{Predicting default and lending decisions} Suppose the planner is a lender. For each credit applicant, the \emph{outcome} of interest to the lender could be the expected profit from this loan. \emph{Actions} available to the lender include the terms of the loan, including whether to extend a loan at all. The lender not only wants to predict the expected profit (e.g. from credit risk); they also need to make measurements that allows them to structure loan terms such that the profit from the loan can be improved, e.g. the credit applicant's sensitivity to interest rate changes. \todo{e.g. Latent state total monthly payment credit applicant can afford. This is an example, where the planner doesn't necessarily have the individual's interest in mind. Both individual and planner agree that Y is advancing their long-term interest.}
}

\section{Illustrative example with two latent states}\label{sec:comp_ex}

In this section, we work out a simple example of the model for illustration. In this instantiation of the model, we assume that all variables are binary. The corresponding graphical model is displayed in Figure~\ref{fig:two_state_graph}.

\begin{itemize}[leftmargin=*]
    \item \emph{States}. There are two latent states $X_1, X_2$ distributed as independent Bernoulli random variables with failure rate $p< 0.5$. That is, we have $X_1, X_2 \sim$ Bernoulli$(1-p)$. 
    \item \emph{Outcome}. The outcome of interest is $Y := X_1 \land X_2 $, where $\land$ denotes the logical and.
    \item \emph{Measurements}. The space of measurements $\calM$ is all Boolean functions of $(X_1,X_2)$. The measurement budget is $B=1$. For the purposes of this example, we focus on the following 3 measurements:\[M_1 := X_1; \quad  M_Y := X_1 \land X_2; \quad  M_\piv := X_1 \land \neg X_2. \]
    
By the symmetry of the example, the other plausible measurements such as $X_2$ and $X_2 \land \neg X_1$ follow similar calculations. We call $M_\piv$ a \emph{pivotal} measurement, which indicates that a particular state is pivotal for changing the outcome.\footnote{In Section~\ref{app:example_pivotal}, we describe examples of real world pivotal measurements , such as ``tell-tale'' symptoms of diseases that exclude other conditions and suggest a clear treatment path, and contrast them with non-pivotal versions.}
    \item \emph{Actions}. The actions are $\calA :=  \{[X_1 \leftarrow 1] , [X_2 \leftarrow 1],  \emptyset\}$. The cost of action is fixed for $[X_i\leftarrow 1]$ at $c > 0$.
    \item \emph{Utility}. Planner's utility from outcome $Y$ is $u(y) = y$.
\end{itemize}

\begin{figure}[tbp]
    \centering
\begin{tikzpicture}[->,shorten >=1pt,auto,node distance=15mm,
                    thick,main node/.style={circle,draw,font=\small,inner sep=2pt}]

  \node[main node] (Y) {$Y$};
  \node[main node] (A1) [above left=7mm of Y] {$X_1$};
  \node[main node] (A2) [below left=7mm of Y] {$X_2$};
  \node[main node] (M1) [left=15mm of A1] {$M_1$};
  \node[main node] (M2)  [below= 1mm of M1] {$M_Y$};
  \node[main node] (M3) [left=14mm of A2] {$M_\piv$};

  \path[every node/.style={font=\small}]
    (A1) edge node [right] {} (Y)
        edge node [right] {} (M1)
        edge node [right] {} (M2)
        edge node [right] {} (M3)
    (A2) edge node [right] {} (Y)
        edge node [right] {} (M2)
        edge node [above, midway] {$\neg$} (M3)
    ;
\end{tikzpicture}

    \caption{Binary variable model with 2 latent states. Arrows indicate logical addition unless otherwise stated.}
    \label{fig:two_state_graph}
\end{figure}
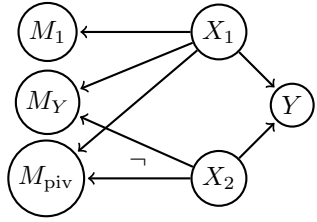

\begin{figure*}[tbp]
    \centering
    \includegraphics[width=0.5\textwidth]{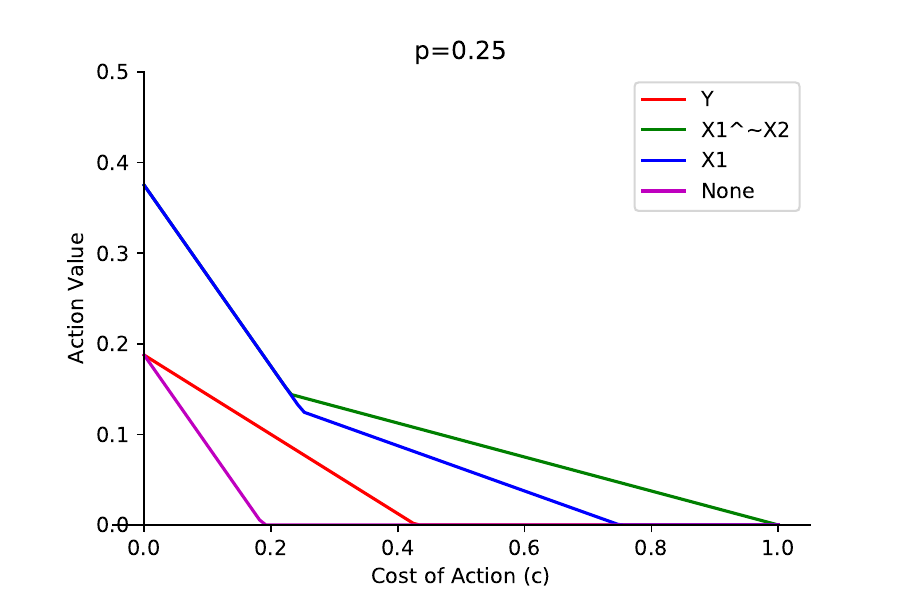}%
    \includegraphics[width=0.5\textwidth]{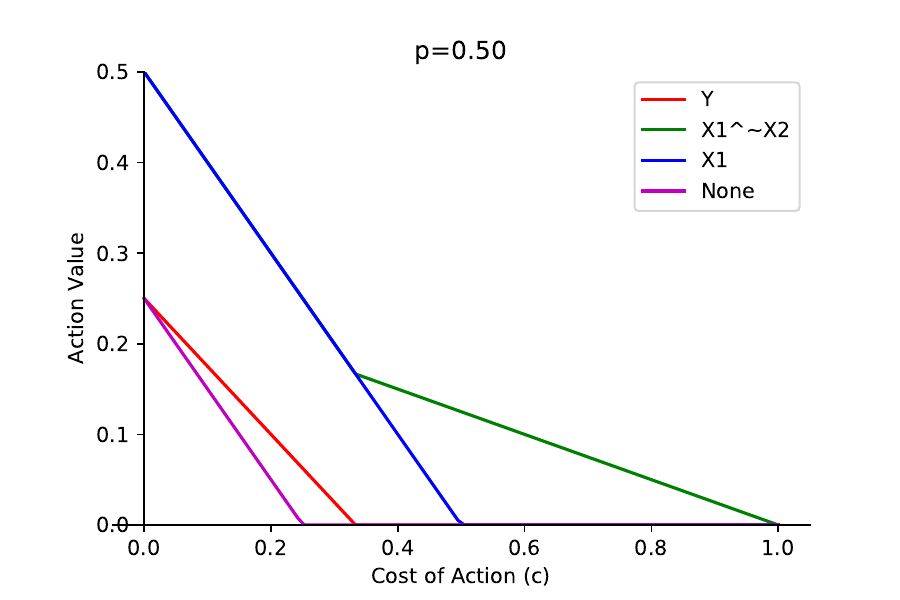}
    \caption{Action value against action cost for 3 measurements: $M_Y$ (Highest prediction value), $M_\piv$, and $M_1$. The action value of making no measurements is included as a baseline. Failure rate $p$ is set to $0.25$ in left plot, and to $0.5$ in the right plot.}
    \label{fig:two_state_act_val}
\end{figure*}

\subsection{Prediction value}

Suppose the hypothesis class $H$ is any (potentially randomized) Boolean function on $\{0,1\}$ and we consider the 0-1 loss. For notational brevity, we write $q= 1-p$ whenever necessary. By elementary calculations, we know:
\begin{itemize}[leftmargin=*]
    \item The best predictor of $Y$, given $M_1$ is to predict $1$ when $M_1 = 1$ and $0$ when $M_1$ = 0. 
    \item The best predictor of $Y$, given $M_Y$ is to predict $M_Y$. 
    \item The best predictor of $Y$, given $M_\piv$, is 
    \begin{itemize}
        \item If $q^2 > p$: predict $0$ when $M_\piv = 1$ and predict $1$ when $M_\piv = 0$;
        \item Otherwise: always predict $0$. \lledit{(See footnote\footnote{When $M_\piv = 1$, $Y$ must be $0$, that is $\Pr(Y=0\mid M_\piv=1)=1$ but in the case when $M_\piv = 0$, we have $\Pr(Y=1\mid M_\piv) > \Pr(Y=0 \mid M_\piv)$ if and only if $q^2 > p$.}.)}
    \end{itemize}
\end{itemize}  

We may compute the (negative) prediction value of each measurement as follows. This is none other than the expected loss of the respective optimal predictor:
\begin{align*}
     -V\pred(M_1) 
        &= 1- \Pr(X_1 = 1, X_2 = 0) \\
        &= 1- pq, \\ 
    -V\pred(M_Y) &= 1 \\
    -V\pred(M_\piv) &= \max(1-p, 1-q^2).
\end{align*}

Hence, ranking the measurements by prediction value, we have $M_Y \succeq M_1 \succeq M_\piv$.

\subsection{Action value}

We turn to the intervention task. To compute action values, we analyze the best action policies given each measurement. First consider $M_\piv$. In this case, the best action policy is:  
 \begin{itemize}[leftmargin=*]
        \item If $M_\piv = 1$, we have $X_1 = 1$ and $X_2 = 0$. The best action is $[X_2 \leftarrow 1]$.
        \item If $M_\piv = 0$, the best action depends on $c$ and $p$. If $\Pr(X_2\land \neg X_1 = 1 \mid M_\piv = 0) > c$, the best action is $[X_1 \leftarrow 1]$. Otherwise, the best action is to do nothing.
    \end{itemize}
The key takeaway is that $M_\piv$ allows the planner to take the action that is \emph{pivotal} for improving the outcome. This both maximizes the utility gain from successfully improving the outcome, and minimizes the cost of taking actions.

Performing similar analyses for $M_Y$ and $M_1$, we find that the action value of each measurement is:
\begin{align*}
    V\act(M_1) &= \max(0, pq - pc) + \max(0, pq - qc)\\ 
    V\act(M_Y) &= \max(0, pq - (1-q^2)c) \\
    V\act(M_\piv) &= \max(0, pq(1-c)) + \max(0, pq - (1-pq)c.
\end{align*}

In Figure~\ref{fig:two_state_act_val}, we plot the action value against cost $c$ for two different $p$ parameters . First we observe that $M_\piv$ has the highest action value regardless of action costs.

We also observe that when actions are very costly, the measurement $M_Y$ is no longer helpful in terms of action utility, i.e., $V\act(M_Y)=0$. $M_Y$ corresponds to perfect knowledge of the outcome $Y$, which is typically not possible in reality. Yet, even under this favorable assumption, we see that knowing the outcome has rather limited utility for effective intervention.

In contrast, $M_\piv$ and $M_1$, which as we recall have lower prediction value, help to inform good action policies. When actions are low-cost, all measurements have positive action value. When actions are very low-cost, $M_\piv$ and $M_1$ have the same action value and their advantage over $M_Y$ is even greater than when actions are costly.

Ranking the measurements by action value, we have $M_\piv \succeq M_1 \succeq M_Y$. In this case, the order is completely reversed from the ranking by prediction value. In Section~\ref{sec:theorems}, we will see that this is an instance of a more general phenomenon.

\subsection{Examples of pivotal measurements}\label{app:example_pivotal}

We introduced the \emph{pivotal measurement} in the the preceding sections as a mathematical construct, that is $M_\piv := X_1 \land \neg X_2$, and showed that it attains the highest action value among all measurements in a simple example. In this section, we consider plausible real world analogues for the pivotal measurement.

A pivotal measurement indicates that a particular intervention is sufficient for improving an individual's outcome. 
In the education realm, being a recipient of a Federal Pell grant could be thought of as a natural example of a pivotal or near-pivotal measurement. It indicates that one has exceptional financial need \emph{and} that one is in good academic standing. It is particularly helpful for targeting an educational intervention that addresses financial opportunity for the purposes of improving educational outcome, since Pell grant recipients are more likely to benefit from this intervention than a student who has financial need but may not be in good academic standing, or a student who is in good academic standing but may not have financial need. 

In the context of medical diagnosis and treatment, most physiological measurements can point to multiple possible ailments. Pivotal measurements correspond to what we think of as `tell-tale' symptoms that suggest a clear treatment path. For example, a patient who has a cough might have either COVID or the common cold, whereas losing the sense of taste and smell is a more distinguishing symptom of COVID \citep{dawson2021loss}, giving the clinical practitioner a higher degree of certainty that the patient should be treated for COVID. Another example of a pivotal measurement in the context of disease is \emph{erythema migrans}, the characteristic rash used for the early diagnosis of Lyme disease, commonly known as the bull's eye rash \citep{aucott2009diagnostic}. Though not all Lyme disease patients develop the rash, the presence of the bull's eye pattern is considered to be more indicative than a blood test that a patient should be treated. 

\section{Main results: Boolean outcome functions}\label{sec:theorems}

In this section, we consider a general setting, where  $Y$ is a Boolean function of $s$ States, $X_1, \cdots, X_s$. That is, $Y:\{0,1\}^s \to \{0,1\}$. This is a simplified setting where the states, and the outcome, can be either good or bad; yet the space of Boolean functions is sufficiently rich to capture wide range of interactions between states and outcome. Thus we focus on such functions for our theoretical analysis.

Suppose that we can measure $M$, any Boolean function of $X_1, \cdots, X_s$; in other words, $\calM$ is the set of all Boolean functions over $\{0,1\}^s$. The set of possible actions is $\calA = \bigcup_{i=1, \cdots, s, x \in \{0,1\}}\{[X_1 \leftarrow x]\} \cup \emptyset$. In words, the planner can set the value of any state to $0$, or $1$, or do nothing.

\subsection{Prediction and action for single measurement}

We first consider the case where the measurement budget is $B=1$. In Proposition~\ref{prop:pivotal}, we work through an example for symmetric and monotone outcome function $Y$ where the states are i.i.d. Bernoulli random variables. 
Then in the main result (Theorem~\ref{thm:main_single}), we give the sufficient and necessary condition for outcome prediction to have maximum action value for any Boolean $Y$. The condition results in a very constrained outcome model, where a single latent state always improves the outcome (see Definition~\ref{def:full_im}). In other words, outcome prediction never has the optimal action value except in highly degenerate models.

The following illustrative result generalizes the example in Section~\ref{sec:comp_ex}. Proposition~\ref{prop:pivotal} gives an explicit expression for a measurement $M^*$ that has higher action value than the measurement $M_Y$ that perfectly tracks the outcome $Y$. $M^*$ generalizes the pivotal measurement that was introduced in the previous section. The proof proceeds by deriving explicit algebraic expressions for action values in terms of model parameters.

\begin{proposition}[Construction of measurement with high action value]\label{prop:pivotal}
    Suppose $Y$ is a symmetric and monotone Boolean function of $s$ States, $X_1, \cdots, X_s$, which are i.i.d. Bernoulli$(1-p)$ random variables. We can measure $M$, any Boolean function of $X_1, \cdots, X_s$, and take any action $a \in \calA$ for a fixed cost $c \in (0,1)$. Then, the following measurement $M^*$ has higher action value than $M\pred = Y$ for any $c\in (0,1)$:
    \begin{align*}
       & M^*(a_1, \cdots, a_s) = 1 \iff \\ &Y(a_1, \cdots,a_{i-1}, 1-a_i, a_{i+1}, \cdots, a_s) \\ &-  
        Y(a_1, \cdots,a_{i-1}, a_i, a_{i+1}, \cdots, a_s) =1,
    \end{align*}
    for any $i = 1, \cdots, s$. Moreover, the inequality is strict for all but univariate $Y$.
\end{proposition}

\begin{proof}[Proof of Proposition~\ref{prop:pivotal}]
 WLOG, we consider monotonically non-decreasing $Y$. Note that in this case $Y$ is in the class of threshold functions, that is, $Y = \Ind{\sum_{i=1}^{s}X_i \ge k}$.

     We use the following notation $\{X_i~\pivotal\}$ to denote the event 
     \begin{align*}
       \{&Y(X_1, \cdots,X_{i-1}, 1-X_i, X_{i+1}, \cdots, X_s) \\
       - &  Y(X_1, \cdots,X_{i-1}, X_i, X_{i+1}, \cdots, X_s) =1\}.
     \end{align*}

        Fix any $i\in [s].$ Recall that $p$ is the probability $X_i = 0$. Denote $q = 1-p$.
        
    Given  $M\pred = Y$, the best single action (to improve $Y$) is as follows:
    \begin{itemize}
        \item If $Y=1$, do nothing.
        \item If $Y=0$, 
        \begin{itemize}
            \item Do $a = [X_i\leftarrow 1]$, if $Q > c$, where \begin{align*}
                Q &= \E[Y^{[X_i\leftarrow 1]}, Y=0] -  \E[Y\mid Y=0] \\
                &= \Pr(X_i~\pivotal \mid Y=0).
            \end{align*}
            \item Do nothing otherwise.
        \end{itemize}
    \end{itemize}
    
    Therefore we have
    \begin{align*}
        V\act(Y) &= \max(0,\Pr(X_i~\pivotal) - c\cdot \Pr(Y=0)),\\
        \text{where}~& \Pr(X_i~\pivotal) = {s-1 \choose k-1} p^{s-k+1}q^{k-1}, \\
        \text{and}~& \Pr(Y = 0) = \sum_{m=k}^{m=s} { s \choose m} p^m q^{s-m}.
    \end{align*}
    
    On the other hand, given $M^*$, the best single action is:
    \begin{itemize}
        \item If $M^* = 1$, $X_i$ is $\pivotal$
        \begin{itemize}
            \item do $a = [X_i\leftarrow 1]$, since $1 > c$
            \item else, do nothing.
        \end{itemize}
        \item If $M^* = 0$, $X_i$ is not $\pivotal$
        \begin{itemize}
            \item do $a = [X_j\leftarrow 1]$, for some $j\ne i$, if $\Pr(X_j~\pivotal\mid X_i~\isnot~\pivotal) > c$
            \item else, do nothing.
        \end{itemize}
    \end{itemize}
    
    Thus we have
    \begin{equation}
        V\act(M^*) \geq \Pr(X_i~\pivotal)\cdot (1-c).
    \end{equation} 
    Note that $\Pr(X_i~\pivotal) \le \Pr(Y=0)$. Thus we have shown that $V\act(M^*)\geq V\act(Y).$
    
    Moreover for $s\ge k > 1$, $\{X_i~\pivotal\} \ne \{X_j~\pivotal\}$ so $\Pr(X_i~\pivotal) < \Pr(Y=0)$. Thus we have \[V\act(M^*) > V\act(Y).\]
\end{proof}

The remaining goal of this section is to generalize the above proposition to arbitrary Boolean outcome function $Y$. To do so, we introduce two new definitions. 
\begin{definition}\label{def:full_im}
    An outcome $Y$ is \textbf{fully improvable} if for any $x_1, \cdots, x_s$ where $\Pr(X_1=x_1, \cdots, X_s = x_s ) > 0$ and $Y(x_1, \cdots, x_s) = 0$, there exists $i, x$, s.t.
\begin{equation*}
    \Pr(Y^{[X_i\leftarrow x]}=1\mid X_1=x_1, \cdots, X_s = x_s)=1.
\end{equation*}
\end{definition}
\begin{definition}
The 
action $[X_{i^*}\leftarrow x]$ 
is \textbf{sufficient} for improving $Y$ if $\forall x_1, \cdots, x_s$ s.t. $\Pr(X_1=x_1, \cdots, X_s = x_s) > 0$,
\begin{equation*}
    \Pr(Y^{[X_{i^*}\leftarrow x]}=1\mid X_1=x_1, \cdots, X_s = x_s)=1.
\end{equation*}
\end{definition}
Full improvability is a strong condition on the outcome $Y$ which states that there always exists a single action on the latent states that improves $Y$ almost surely. This single action can in general depend on the latent states. Full improvability is already a restrictive condition: threshold functions, $\Ind{\sum_{i=1}^{s}X_i \ge k}$, are not fully improvable for $k>1$. 

The existence of a sufficient action is an even stronger condition which states that the same single action improves $Y$ almost surely across all realizations of the latent states. We note that having a sufficient action indicates that the outcome $Y$ is fully improvable, but the former does not necessarily imply the latter. For example, the parity function of $s$ Boolean variables is fully improvable, but it does not have a sufficient action. On the other hand, we note that every monotone and fully improvable $Y$ must have a sufficient action.

Recall that having a sufficient action means that there's a single action that improves $Y$ whenever $Y=0$ regardless of the configuration of the latent states; that is, one treatment helps all equally. We now show that predicting the outcome is optimal for taking actions if and only if the strong and likely unrealistic condition---of having a sufficient action---holds. The proof of the forward implication proceeds by constructing a measurement $M$ such that it has higher action value than $Y$, whenever $Y$ does not have a sufficient action. We illustrate the proof idea in Figure~\ref{fig:proof_idea}.

\begin{theorem}[Outcome prediction and maximum action value]\label{thm:main_single}
    Let $Y(X_1, \cdots, X_s)$ be an $s$-dimensional Boolean function such that $Y \not\equiv 0$. If $Y$ does not have a sufficient action, there exists $M(X_1, \cdots, X_s)$ such that $V\act(M) > V\act(Y)$ for $c<1$. If $Y$ has a sufficient action $[X_{i^*}\leftarrow x]$,
    then $V\act(Y)$ is maximal for any cost $c$. 
\end{theorem}

\begin{proof}
We prove the first direction, that is, we assume $Y$ does not have a sufficient action. Suppose the best action given $Y=0$ is $do(X_1\leftarrow x)$, WLOG. Since the best action given $Y=1$ is $\emptyset$, the action value of $Y$ is
\begin{equation*}
    V\act(Y) = \max(0, \Pr(Y=0 \cap Y^{[X_{i^*}\leftarrow x]}=1) - \Pr(Y=0)\cdot c)
\end{equation*}
Let $M$ be s.t. $\{M=0\}= \{Y=0\} \cap \{Y^{[X_{i^*}\leftarrow x]}=1\} $. The action value of $M$ is 
\begin{equation*}
    V\act(M) \ge \max(0, \Pr(Y=0 \cap Y^{[X_{i^*}\leftarrow x]}=1)\cdot (1-c)).
\end{equation*}
By assumption, we have that $\Pr(Y=0) > \Pr(M=0)$. Therefore, for any $c<1$, we have $V\act(M) > V\act(Y)$.

Now, for the other direction. Recall that the action value of $Y$ is $\Pr(Y=0)\cdot (1-c)$. Consider some measurement $M$ that is $(X_1, \cdots, X_s)$-measurable, and suppose the optimal action policy given $M$ is:
\[a(M) = \begin{cases}
[X_{i} \leftarrow x_i] &\text{ if } M=0 \\
[X_{j} \leftarrow x_j] &\text{ if } M=1
\end{cases}.\]
Then the action value of $M$ is $V\act(M)$
\begin{align*}
    &= \max(0, \Pr(Y=0, M=0, Y^{[X_i\leftarrow x_i]} =1) - \Pr(M=0)\cdot c)\\
    &\quad + \max(0, \Pr(Y=0, M=1, Y^{[X_j\leftarrow x_j]} =1) - \Pr(M=1)\cdot c)\\
    &\leq \max(0, \Pr(Y=0, M=0) - \Pr(M=0, Y=0)\cdot c)\\
    &\quad + \max(0, \Pr(Y=0, M=1) - \Pr(M=1, Y=0)\cdot c)\\
    &= \Pr(Y=0, M=0)\cdot (1-c) + \Pr(Y=0, M=1) \cdot (1-c)\\
    &= \Pr(Y=0)\cdot (1-c).
\end{align*}
This shows that $Y$ has the maximal action value among all measurements.
\end{proof}

\begin{figure}[tbp]
    \centering
    \resizebox{0.5\textwidth}{!}{
    \begin{tikzpicture}
    \draw[thick] (-4,0) arc (180:0:4 and 2.5);
    \node at (-3.5, 2) {$Y=0$};

    \draw[thick] (4,0) arc (0:-180:4 and 2.5);
    \node at (3.5, 2) {$Y=1$};

    \begin{scope}
    \clip (0,0) circle (3.5);
    \fill[gray!50] (0,0) circle (2);
    \end{scope}
        \begin{scope}
    \clip (0,0) circle (2);
    \fill[blue!30] (-2,2) rectangle (0,-2.5);
    \end{scope}
    
    \draw[thick] (0,0) circle (2);
        \node[fill=white, inner sep=2pt] at(0, 0)  {$Y^{[X_{i^*} \leftarrow x]}=1$};

    \draw[thick] (0,2.5) -- (0,-2.5);

    \draw[thick] (-1,-1) -- (-1, -2.8);
    \node at (-1, -3) {$M=0$};

    \end{tikzpicture}
    }
    \caption{Illustration of forward implication in proof of Theorem 4.2. The large oval depicts the measure space over latent states $(X_1, \cdots, X_s)$. The blue shaded region depicts the subset of latent states where $M$ takes value $0$ and is the intersection of two regions---the region where $Y=0$ and the region where $Y^{[X_{i^*} \leftarrow x]}=1$.}\label{fig:proof_idea}
\end{figure}
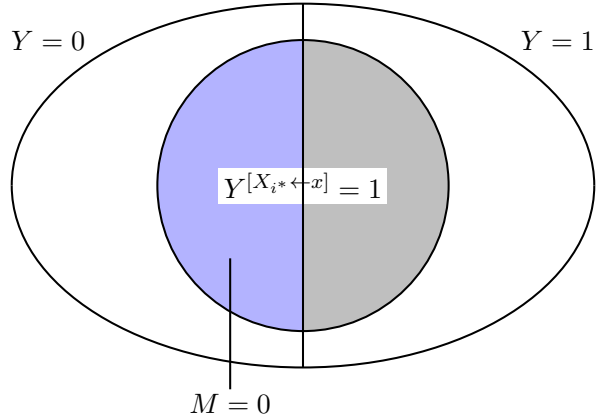

\subsection{Prediction and action for a measurement set}

In this section, we consider measurement sets of size $B > 1$ and we prove a generalization of the second implication in Theorem~\ref{thm:main_single}---that $Y$ is typically not part of a set of measurements that together maximizes action value.

As we turn our consideration from single measurements to  measurement sets, information that is conveyed from certain measurements may become redundant. On the question of whether $Y$ is an actionable measurement when used in combination with other measurements, we would therefore like to focus on measurement sets where $Y$ is \textbf{non-redundant}, defined as follows.
\begin{definition}\label{def:non_redun}
    Consider $Y\cup S$, a size-$B$ measurement set containing $Y$ for $S$ such that $|S| = B-1$ and $S\subseteq \calM$. We say that $Y$ is \textbf{non-redundant} with respect to $S$ \lldelete{in $Y\cup S$} if there exists $s \in \{0,1\}^{B-1}$ such that the best action when $Y=0, S=s$ is not $\emptyset$. We call the 
    set of such $s$ the \textbf{$Y$-relevant set} with respect to $S$:
    \begin{equation*}
        \{s \in \{0,1\}^{B-1}: a^*(Y=0, S=s) \neq \emptyset\}.
    \end{equation*}
\end{definition}

Note that the best action when $Y=1$ is always $\emptyset$, so by Definition~\ref{def:non_redun}, the non-redundant set of $Y$ with respect to $S$ \lldelete{in $Y\cup S$}  is where optimal action is dependent on $Y$ conditioning on $S$.
We also extend the notion of sufficient action to subsets of the probability space.
\begin{definition}\label{def:suff_act_cond}

For any $\calS$-measurable set $\calF$, the action $[X_{i^*}\leftarrow x]$ 
is \textbf{sufficient for improving $Y$ on $\calF$} if $\forall x_1, \cdots, x_s$ s.t. $\Pr(X_1=x_1, \cdots, X_s = x_s\mid \calF) > 0$,
\begin{equation*}
    \Pr(Y^{[X_{i^*}\leftarrow x]}=1\mid X_1=x_1, \cdots, X_s = x_s)=1.
\end{equation*}
\end{definition}

In the following theorem, we show that $Y$ cannot be a element of an optimal measurement set, unless  $Y$ is a redundant measurement, or a strong condition is satisfied: that $Y$ has a sufficient action whenever it is non-redundant.

\begin{theorem}[Action value of measurement sets containing the outcome can be improved]\label{thm:set}
 Let $Y(X_1, \cdots, X_s)$ be an $s$-dimensional Boolean function\lldelete{such that $Y \not\equiv 0$}. \lledit{Consider $Y\cup S$, the size-$B$ measurement set where $Y$ is non-redundant. Suppose there exists $\bar{s}\in \{0,1\}^{B-1}$ in the $Y$-relevant set with respect to $S$ such that $Y$ does not have a sufficient action on $\{S=\bar{s}\}$.}

 Then, there is a measurement $M^*$ such that 
 \begin{equation*}
     V\act(Y\cup S) < V\act(M^* \cup S).
 \end{equation*}
\end{theorem}
\begin{proof}
   By assumption, there exists $\bar{s} \in \{0,1\}^{B-1}$ such that $a^*(Y=0, S=\bar{s}) = [X_i \leftarrow x]$ and $[X_i \leftarrow x]$ is not a sufficient action for improving $Y$ on $\{S=\bar{s}\}$.

   Let $X_{NI}$ denote the set of state values $(x_1, \cdots, x_s)$ where $S=\bar{s}$ and the action $[X_i \leftarrow x]$ does not improve $Y$, that is, \begin{equation*}
    \Pr(Y^{[X_i\leftarrow x]}=1\mid X_1=x_1, \cdots, X_s = x_s)=0.
\end{equation*} Since $[X_i \leftarrow x]$ is not a sufficient action on $\{S=\bar{s}\}$, we must have $\Pr((X_1, \cdots, X_s)\in X_{NI}) > 0$.
    
    Construct a new measurement $M^*$ such that 
    \begin{equation*}
        M = \begin{cases}
            1 &\text{ if } Y=1 \text{ or } (X_1, \cdots, X_s) \in X_{NI} \\
            0 &\text{o.w.}
        \end{cases}
    \end{equation*}

    Compare the best action policy under $Y\cup S$ and $M^*\cup S$. The best action policy changes only on the set $\{(X_1, \cdots, X_s) \in X_{NI} \}$, where the planner now takes no action instead of $[X_i \leftarrow x]$. The action value is therefore improved by $\Pr((X_1, \cdots, X_s)\in X_{NI})\cdot c>0$. We have shown that the measurement set $M^* \cup S$ has strictly higher action value than $Y\cup S$. 
\end{proof}

In Theorem~\ref{thm:set}, recall that we assumed $Y$ does not have a sufficient action on some element of the $Y$-relevant set, which implies and is stronger than the condition that $Y$ does not have a sufficient action overall. The following example shows that this assumption is necessary. Consider a slightly modified outcome model from Section~\ref{sec:comp_ex}: there are two binary latent states where $Y = X_1\land X_2$, and the marginal distribution of $X_1, X_2$ is Bernoulli but we have $\Pr(X_1=X_2=0)=0$. Here, $Y$ does not have a sufficient action and yet $\{Y, X_1\land \neg X_2\}$ is an optimal measurement set (maximum action value among all measurement sets).

\section{Further Related Work}\label{sec:related}

\paragraph{Heterogeneous causal effects and policy evaluation} 
\citet{athey2016recursive,shalit2017estimating,wager2018estimation} have examined the estimation of heterogeneous treatment effects from observational data. Furthermore, the area of off-policy learning and optimization \citep{manski2004statistical,zhao2012estimating,dudik2014doubly, kallus2018policy, athey2021policy} studies average causal outcomes under personalized treatment assignment policies, exemplified in studies focusing on job training interventions \citep{kitagawa2018should, knaus2022heterogeneous}. The framework of policy optimization is an alternative framework for algorithmic decision making that precludes the need for outcome predictors and human-in-the-loop decision making; it requires data about treated and untreated outcomes, and the treatment policy under which data was collected. Typically, the estimation of heterogenous treatment effects is limited to scenarios where only a single treatment (either discrete or continuous) is considered, without delving into the problem of diagnosing multiple causal factors. Going beyond randomized controlled trials, adaptive interventions involving multiple assignment strategies has become increasingly popular in the clinical application domain \citep{collins2007multiphase,montoya2022efficient}.
Though we have similar goals of finding optimal personalized treatment assignments---called an ``action policy'' in the current work,  personalization in this line of work depends on given covariates, whereas the current model examines the choice of what covariates to measure under a measurement budget.

\paragraph{Recourse and strategic action in machine learning}
A rich literature has developed over recent years on the topic of \emph{recourse}---that is, how individuals subject to an adverse decision based on a machine learning model might change their feature values to achieve a more favorable decision in the future \citep{ustun2019actionable,verma2020counterfactual,ross2021learning, karimi2022survey}. 
The research emphasizes the need for explanations that highlight mutable and more easily changeable features to guide individual action \citep{joshi2019towards, barocas2020hidden, karimi2021algorithmic}. While this work shares a common motivation with the present paper---to help individuals to achieve desired outcomes rather than just predict likely outcomes---it differs in two crucial ways. First, the work on recourse is specifically focused on the actions that can be taken by decision subjects, whereas we are concerned with the actions available to a social planner who is generally seeking to achieve positive societal impact. Secondly, while recourse focuses on altering the decisions output by a machine learning model, we are concerned with actions that  affect the likelihood of the actual outcome of interest, not merely a model's predictions.

The growing body of research on \emph{strategic classification} aims to assess how decisions subjects might adapt their behavior in light of a machine learning model making decisions \citep{bruckner2011stackelberg,Hardt:2016:SC,kleinberg18investeffort,Hu2019disparate,Milli2019social,liu2020disparate}. This line of work explores the concept of gaming, where individuals manipulate input features to improve model predictions without necessarily affecting the underlying property. Prior work has demonstrated that limiting strategic behavior along these lines requires causal modeling \citep{miller2020strategic,shavit2020causal}.This area of work focuses on designing the right incentives within an ML model to cause decision subjects to behave as the social planner might like them to behave. In contrast, the current work is focused on designing the measurements that help the social planner achieve its interventional goals directly.

\section{Conclusion}

In this paper, we studied the gap between outcome prediction and intervention in a probabilistic graphical model of outcomes, states, actions and measurements. By distinguishing between the utility of a measurement for accurate prediction and for effective intervention, we show that outcome prediction almost never leads to an optimal measurement or an optimal measurement set for interventions.
Our result is framed theoretically at a general level, to provide a language for reasoning about predictions and actions beyond the specifics of any one domain.

On the practical front, this theoretical investigation contributes to the discourse around the actionability of risk prediction in the education domain and beyond \citep{liu2023reimagining}, by specifying the limited conditions under which single outcome or risk prediction is compatible with interventional goals, and may be of interest to applied machine learning practitioners broadly. Recent work by \citet{saxena2023rethinking} has highlighted the fraught multiplicity of predicted risk notions. Further research might look into designing and predicting multiple actionable risk factors that incorporate knowledge of available interventions, as has been implemented in a data-driven student success program at Georgia State University with positive results \citep{renick2020predictive}, deconstructing the time dimension of risk \citep{xing2019dropout}, as well as recognizing conditions under which one might decide not to construct or implement outcome predictors \citep{garcia2020no,wang2022against}.

\newpage

\bibliography{mybib}

\end{document}